\documentclass[conference]{IEEEtran}
\usepackage{amsmath,amssymb,amsfonts}
\usepackage{graphicx}
\usepackage{subfigure}
\usepackage{textcomp}
\usepackage{xcolor}
\usepackage{mdframed}
\usepackage[numbers]{natbib}
\usepackage[ruled,vlined]{algorithm2e}
\usepackage{algpseudocode}
\usepackage{enumitem,kantlipsum}
\usepackage{amsthm}

\def\BibTeX{{\rm B\kern-.05em{\sc i\kern-.025em b}\kern-.08em
    T\kern-.1667em\lower.7ex\hbox{E}\kern-.125emX}}

\colorlet{shadecolor}{yellow}

\newcommand{\ainvase}{Assorted}
\newcommand{\our}{assorted}
\newcommand{\ourU}{U_{\textit{ASE}}}

\newtheorem{lemma}{Lemma}
\begin{document}

\title{VASE: Variational \ainvase{} Surprise Exploration for Reinforcement Learning\\
%{\footnotesize \textsuperscript{*}Note: Sub-titles are not captured in Xplore and
%should not be used}
%\thanks{Identify applicable funding agency here. If none, delete this.}
}

\author{\IEEEauthorblockN{ Haitao Xu, Brendan McCane, Lech Szymanski}
\IEEEauthorblockA{\textit{Department of Computer Science} \\
\textit{University of Otago}\\
Dunedin, New Zealand \\
\{haitao,mccane,lechszym\}@cs.otago.ac.nz}

}

\maketitle

\begin{abstract}
Exploration in environments with continuous control and sparse rewards remains a key challenge in reinforcement learning (RL). Recently, surprise has been used as an intrinsic reward that encourages systematic and efficient exploration.  We introduce a new definition of surprise and its RL implementation named Variational \ainvase{} Surprise Exploration (VASE).  VASE uses a Bayesian neural network as a model of the environment dynamics and is trained using variational inference, alternately updating the accuracy of the agent's model and policy.  Our experiments show that in continuous control sparse reward environments VASE outperforms other surprise-based exploration techniques. 
\end{abstract}


\section{Introduction}

Reinforcement learning (RL) trains agents to act in an environment so as to maximise cumulative reward.  The resulting behaviour is highly dependent on the trade-off between exploration and exploitation.  During training, the more the agent departs from its current policy, the more it learns about the environment, which may lead it to a better policy; the closer it adheres to the current policy, the less time wasted exploring less effective options.  How much and where to explore has an immense impact on the training and ultimately on what the agent learns.  Designing exploration strategies, especially for increasingly complex environments, is still a significant challenge.  

A common approach to exploration strategies is to rely on heuristics that introduce random perturbations into the choices of actions during training, such as $\epsilon$--greedy \cite{sutton1998introduction} or Boltzmann exploration \cite{mnih2015human}.  These methods instruct the agent to occasionally take an arbitrary action that may drive it into a new experience.  Another way is through the addition of noise to the parameter space of the agent's policy neural network \cite{fortunato2017noisy, plappert2017parameter}, which varies the policy itself to a similar random exploration net effect.  These strategies can be highly inefficient because they are a result of random behaviour, which is especially problematic in high dimension state-action spaces (common in discretised continuous state-action space environments) because of the curse of dimensionality.   Random exploration is also extremely inefficient in environments with sparse rewards, where the agent ends up wandering aimlessly through the state-space, learning nothing until (by sheer luck) it chances upon a reward.  Not surprisingly, more methodical approaches were devised,  which provide the agent with intrinsic rewards that encourage efficient exploration.  These intrinsic rewards are derived from computations related to the notion of curiosity \cite{pathak2017curiosity, burda2018large} or surprise \cite{houthooft2016vime, achiam2017surprise}. 

In this paper, we propose a new definition of surprise, which drives our agents' intrinsic reward function. To compute and use this surprise for guiding exploration, we propose an algorithm called VASE (Variational \ainvase{} Surprise Exploration) in a model-based RL framework (see Figure \ref{fig:rl}).  VASE alternates the update step between the agent's policy and its model of the environment dynamics.  The policy is implemented with a multilayer feed-forward (MLFF) neural network and the dynamics model with a Bayesian neural network (BNN~\cite{graves2011practical, blundell2015weight, hinton1993keeping}).  We evaluate the performance of our method against other surprise-driven methods on continuous control tasks with sparse rewards. Our experimental results show VASE's superior performance. 

\begin{figure}
    \centering
    \subfigure[Model-free RL]
    {
        \includegraphics[width=0.4\linewidth]{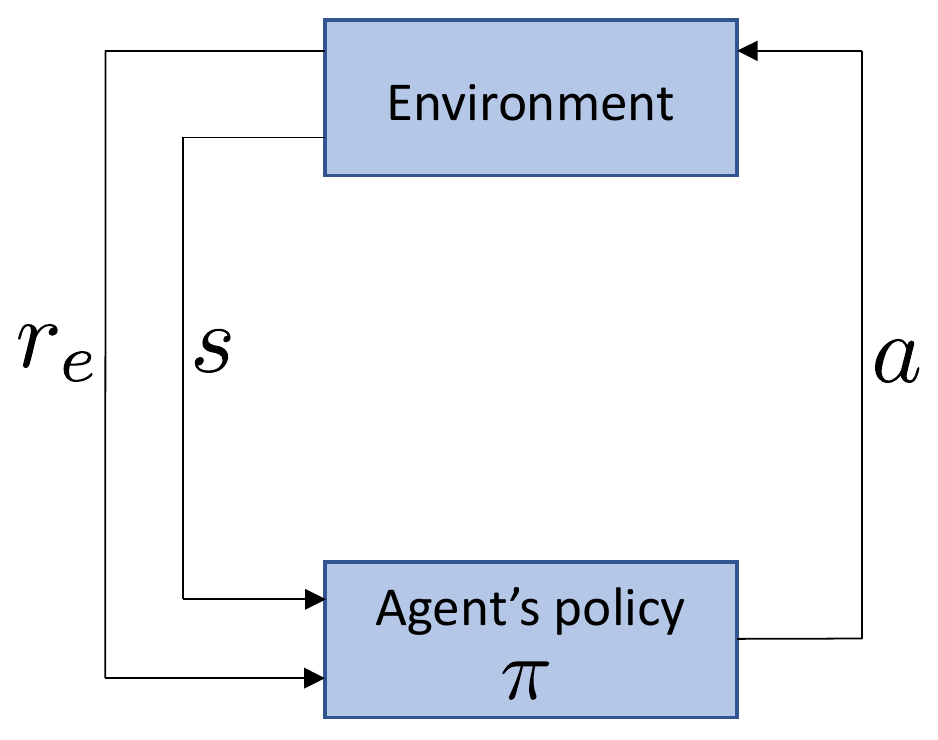}
        \label{fig:tr-rl}
    }
        \subfigure[Model-based/surprise-driven RL]
    {
        \includegraphics[width=0.45\linewidth]{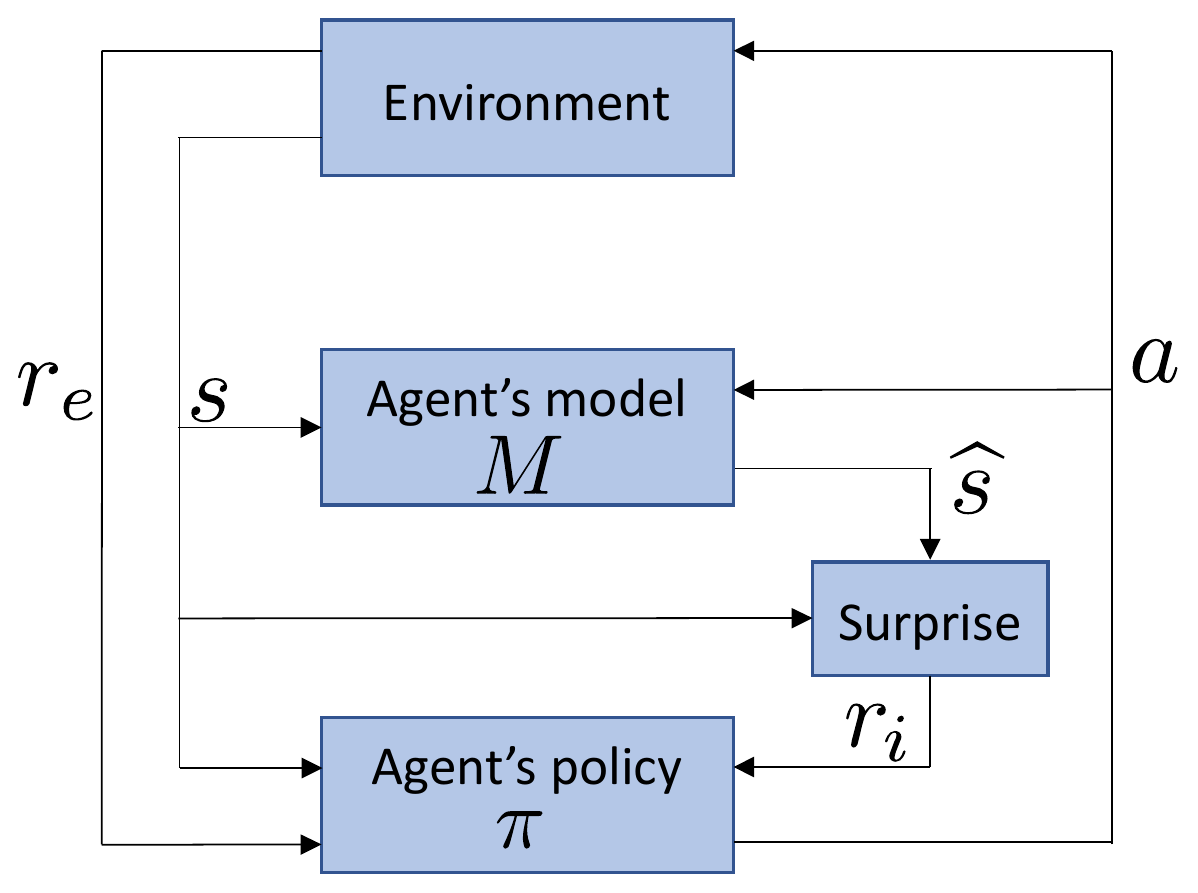}
        \label{fig:sur-rl}
    }
    \caption{Model-free RL vs. model-based/surprise-driven RL with $s$ for state, $r_e$ for the extrinsic (environment-driven) reward, $a$ for action, $r_i$ for intrinsic (agent-driven) reward, $\pi$ for the policy, and $M$ for the model of the environment that makes a prediction of the next state $\hat{s}$.}
    \label{fig:rl}
\end{figure}

\section{Related work}

In RL a finite horizon discounted Markov decision process (MDP) is defined by a tuple $(\mathcal{S, A, P}, r, \gamma, T,\rho_0)$ where: $\mathcal{S}$ is a state set, $\mathcal{A}$ an action set, $\mathcal{P} : \mathcal{S \times A \times S} \rightarrow [0,1]$ is a transition probability distribution, $r:\mathcal{S} \rightarrow \mathbb{R}$ a reward function,  $\gamma \in (0, 1]$ a discount factor, $T$ the horizon and $\rho_0 $ an initial state distribution. A policy $\pi : \mathcal{S \times A}\rightarrow [0, 1]$ gives the probability with $\pi(a|s)$ of taking action $a$ in state $s$. Let $\tau = (s_0 , a_0 , \cdots)$ denote the whole trajectory, $s_0 \sim \rho_0$, $a_t \sim \pi(a_t |s_t )$.  Our purpose is to find a policy $\pi$, modelled by $f_\pi(s_t,\psi)$ with parameters $\psi$, which maximises the expected discounted total return $\mathbb{E}_{\tau}[\sum_{t=0}^{T}\gamma^tr(s_t)]$ -- a discounted sum of all rewards in a fixed horizon $T$.  
\subsection{Intrinsic rewards}

Intrinsic motivation is essential for effective exploration when training the agent in an environment with sparse extrinsic rewards, or no rewards at all.  The overall reward signal is computed as follows:
 $$r(s_t) = r_e(s_t) + r_i(s_t), $$
 where $r_e(s_t)$ represents the extrinsic reward from the environment, $r_i(s_t)$ represents the intrinsic reward computed by the agent.  Even when $r_e(s_t)=0$, the intrinsic reward contributes to the cumulative reward, thus driving the learning process until non-zero extrinsic rewards are found.  There are two broad approaches for encouraging the agent to explore through intrinsic rewards.    
 
The first is the count-based approach \cite{bellemare2016unifying, ostrovski2017count, lopes2012exploration, poupart2006analytic}, which maintains visit counters over all the states.  The intrinsic reward is inversely proportional to the current state's counter, thus rewarding exploration of less frequented states. This approach becomes intractable with the increase of possible states, making it unfit for scenarios that have continuous state-action spaces. The second is surprise-based.

\subsection{Surprise-driven reinforcement learning}
\label{sec-surprise}

In model-based RL the agent maintains one or more models that predict the next state based on the current state and the action about to be taken. There are four possible regimes involving increasing amounts of probabilistic reasoning:
\begin{enumerate}
\item There is a single deterministic model
\item There is one model that produces a distribution over new states
\item There is a distribution of models, each of which is deterministic
\item There is a distribution of models, each of which produces a distribution over states.
\end{enumerate}

As an example, the model for case 1 could be a traditional neural network, case 2 variational auto-encoder, case 3 Bayesian neural network and case 4 Bayesian variational auto-encoder.  In cases 2-3, the outcome is a distribution over states and we are free to choose the most convenient formalism. For this paper we choose case 3 as in \cite{baldi2010bits}. In this case the agent maintains a distribution $P(M)$ over models or hypotheses $M \in \mathcal{M}$, where $M: \mathcal{S \times A} \rightarrow \mathcal{S}$ predicts state $s_{t+1}$, given state $s_t$ and action $a_t$. Furthermore, we have:
\begin{equation}
\label{eq-st1-withM}
  P(s_{t+1} | s_t, a_t, M) = \begin{cases}
    P(M|s_t,a_t), &\text{when } s_{t+1}=M(s_t, a_t)\\
    0, &\text{otherwise}
  \end{cases}
\end{equation}
and
\begin{equation}
\label{eq-st1-noM}
P(s_{t+1} | s_t, a_t) = \int_{\mathcal{M}} P(M|s_t, a_t) P(s_{t+1} | s_t, a_t, M) \partial M.
\end{equation}
The environment is assumed to be stochastic, but we consider the distributions in Equations \ref{eq-st1-withM} and \ref{eq-st1-noM} are subjective agent-based beliefs about the environment, and not the underlying objective truth. Furthermore, these distributions are non-stationary and change as the agent learns more about the environment. Also note that we equate states with observations even in complicated scenarios when the true state of the environment is not directly observed (e.g. the agent has a camera). This is because the goal of the agent in RL is to maximise long-term rewards, and we assume that those rewards are directly observable from the environment.

The distribution in Equation \ref{eq-st1-noM} allows for a number of definitions of surprise, some of which have previously appeared in the literature. In all cases, the accuracy of the distribution, after observing the next state, is used to derive surprise, $U(s_{t+1})$, and forms the basis for intrinsic reward: $r_i(s_t) = \eta U(s_t)$, where $0<\eta \le 1$, and is used to encourage exploration of \textit{unexpected} states \cite{houthooft2016vime, achiam2017surprise,  pathak2017curiosity, mohamed2015variational, schmidhuber1991possibility, chentanez2005intrinsically}. The agent is rewarded for curiosity of the \textit{unknown} as gauged by its model of the environment.  Throughout training, the model is improved to be more accurate (and so less surprised) next time it encounters an already explored state.  The hope is that this methodical approach to exploration will result in a speedier arrival of the agent at the states with non-zero extrinsic reward.  

It is not entirely obvious how best to define the surprise.  One proposed definition is the so-called \textit{surprisal} \cite{tribus1961thermostatics}, which is the negative log-likelihood (NLL) of the next state in RL tasks
\begin{equation}
U_{\text{NLL}}(s^*_{t+1}) = -\log{P(s_{t+1}=s^*_{t+1}|s_t,a_t)},
\end{equation}
where $s^*_{t+1}$ is the observed state at time $t+1$. This notation is a bit clumsy, so we shorten it to:
\begin{equation}\label{eqn_surprisal}
U_{\text{NLL}}({s}_{t+1}) = -\log{P(s_{t+1}|s_t,a_t)},
\end{equation}
which we think is clear and more concise.

Surprisal is intuitive, simple and easy to compute, but it does not capture all the information available, because it only measures the surprise at a single point. An alternative is \textit{Bayesian surprise}, which measures the difference between the prior distribution over the model space at time $t$, and the posterior distribution updated by Bayes' rule with the newly observed state (first used in RL by \citet{storck1995reinforcement}):
\begin{eqnarray}\label{eqn:theta_bayes}
  P(M | s_t, a_t, s_{t+1}) = \frac{P(M | s_t, a_t) P(s_{t+1}|s_t,a_t,M)}{P(s_{t+1}|s_t, a_t)}
\end{eqnarray}
Bayesian surprise is defined as the Kullback-Leiber (KL) divergence between the prior and the posterior beliefs about the dynamics of the environment  \cite{itti2005principled, itti2006bayesian}:
\begin{eqnarray}\label{sur_bayes}
U_{\text{Bayes}}(s_{t+1}) = D_{KL}[P(M|s_t,a_t)||P(M|s_t,a_t,s_{t+1})].
\end{eqnarray}
Bayesian surprise measures the difference between subjective beliefs prior and post an event.  One problem with Bayesian surprise is that the agent does not express surprise until it updates its belief, which is inconsistent with the instantaneous response to surprise displayed by neural data \cite{faraji2016learning}. 

\citet{faraji2016balancing} introduced a modification of Bayesian surprise referred to as the \textit{confidence-corrected surprise} (CC), which measures the difference between the agent's current beliefs about the world, and a naive observer who believes all models are equally likely:
\begin{eqnarray}\label{sur_cc}
U_{\text{CC}}(s_{t+1}) &=& D_{KL}[P(M|s_t,a_t)||P^{\text{flat}}(M|s_{t+1})],
\end{eqnarray}

They also designed a surprise minimization rule to let the agents adapt quickly to the environment, especially in a highly volatile environment. But confidence-corrected surprise was not applied to RL frameworks. Nevertheless, their definition of surprise is similar in spirit to ours.

To use Bayesian surprise for RL tasks, \citet{houthooft2016vime} proposed a surprise-driven exploration strategy called VIME (variational information maximizing exploration). They showed that VIME achieves significantly better performance compared to heuristic exploration methods across a variety of continuous control tasks with sparse rewards. However, to compute each reward, VIME needs to calculate the gradient through a Bayesian neural network (BNN)\cite{graves2011practical, blundell2015weight}, which is used to implement the agent's model. This requires a forward and a backward pass, which leads to slow training speed. \citet{achiam2017surprise} chose the surprisal of new observation as their intrinsic surprise reward. Their experimental results showed that surprisal does not perform as well as VIME, but it runs faster. They also showed that the surprisal includes $L_2$-squared model prediction error, which was first proposed in \cite{stadie2015incentivizing} and later used as the curiosity in  \cite{pathak2017curiosity}. \citet{burda2018large} also chose surprisal-based strategy exploration in large scale RL environments that provide no extrinsic rewards. 

\section{\ainvase{} Surprise for Exploration}

In this paper, we focus on RL environments with continuous control and very sparse extrinsic rewards.  To conquer the numerous shortcomings of existing definitions for surprise discussed in Section \ref{sec-surprise}, and inspired by the idea of confidence-corrected surprise \cite{faraji2016balancing}, we note that in RL the definition of surprise should have the following characteristics:

\begin{enumerate}[wide, labelwidth=!, labelindent=0pt]
  \item subjectivity -- the agent should hold subjective beliefs about the environment captured through $P(M)$; the surprise depends on an agent's belief and this belief can be updated during learning;
  
  \item consistency -- based on its belief, the agent should be more surprised by states with lower likelihood;

  \item instancy -- the agent should be surprised immediately when it observes a new state from the environment, without the need to update its belief first.
\end{enumerate}

In order to address all of the above characteristics, we propose a new definition of surprise, which we refer to as \textit{\our{} surprise for exploration} (ASE): 
\begin{align}\label{eqn_sours}
\ourU(s_{t+1}) = & \mathbb{E}_{M \sim P(\cdot|s_t,a_t)}[-\log{P(s_{t+1}|s_t,a_t,M)}]\nonumber \\ & -\delta H\big (P(M|s_t,a_t)\big ),
\end{align} 
where $\delta$ is a trade-off coefficient. The first term, $\mathbb{E}_{M \sim P(\cdot|s_t,a_t)}[-\log{P(s_{t+1}|s_t,a_t,M)}]$, we call the \our{} surprise term, and the second term, $H\big (P(M|s_t,a_t)\big )$, the confidence term.
 
We will next demonstrate that the \our{} surprise satisfies all three characteristics we mentioned above. Subjectivity comes from the \our{} surprise term in Equation \ref{eqn_sours} because it is an expectation over the agent's belief in the veracity of each of the models. It is also the sum of the Bayesian surprise and the surprisal (hence the name \textit{\our{}}) as shown in the following lemma.
 
\begin{lemma}[Assorted surprise is the sum of Bayesian surprise and surprisal]
  \begin{align*}
\mathbb{E}_{M \sim P(\cdot|s_t,a_t)}[-\log{P(s_{t+1}|s_t,a_t,M)}] =\\ U_{\text{Bayes}}(s_{t+1})+U_{\text{NLL}}(s_{t+1})
  \end{align*}
\end{lemma}
%\haitaoNote{??}
\begin{proof}
  \begin{eqnarray*}\label{suprise-sum}
 &&\mathbb{E}_{M \sim P(\cdot|s_t,a_t)}[-\log{P(s_{t+1}|s_t,a_t,M)}] \\ 
 & = & -\int_{\mathcal{M}}P(M|s_t,a_t)\log{P(s_{t+1}|s_t,a_t,M)}\partial M\\
& = & -\int_{\mathcal{M}}P(M|s_t,a_t) \log{
    \frac{P(M|s_t,a_t,s_{t+1})P(s_{t+1}|s_t,a_t)}{P(M|s_t,a_t)}}\partial M\\
& = & \int_{\mathcal{M}}P(M|s_t,a_t) \log{
    \frac{P(M|s_t,a_t)}{P(M|s_t,a_t,s_{t+1})}}\partial M \\
    &&-
    \int_{\mathcal{M}}P(M|s_t,a_t)\log{P(s_{t+1}|s_t,a_t)}\partial M \\
    & = & D_{KL}[P(M|s_t,a_t)||P(M|s_t,a_t,s_{t+1})] \\
    && - \log{P(s_{t+1}|s_t,a_t)}] \\
    & = & U_{\text{Bayes}}(s_{t+1})\\
    && + U_{\text{NLL}}(s_{t+1}),
\end{eqnarray*} 
\end{proof}

This means that \our{} surprise is subjective due to the contribution of $U_{\text{Bayes}}$.  However, the expectation in Eq. \ref{eqn_sours} does not require evaluation of $P(M|s_t,a_t,s_{t+1})$, so there is no requirement to update the agent's belief in order to compute \our{} surprise thus satisfying the instancy characteristic.  Finally, a less likely state leads to a larger surprise through the negative log likelihood of $P(s_{t+1}|s_t,a_t)$ in the surprisal term. This satisfies the consistency requirement.

The confidence term of Eq. \ref{eqn_sours}, $H(P(M|s_t,a_t))$, is the entropy of $P(M|s_t,a_t)$.  This term was added for confidence correction of \our{} surprise. A confident agent will have a low entropy, and therefore any surprising event according to the \our{} surprise term will remain surprising. Whereas an uncertain agent will have a large entropy and their overall surprise will be reduced because they would be equally surprised by many events.  That is, confident agents are more surprised when their beliefs are violated by unlikely events than uncertain agents.

\expandafter\MakeUppercase \our{} surprise captures the positive elements of both Bayesian surprise and confidence-corrected surprise. It implicitly computes the difference in belief as in Bayesian surprise without needing to update the belief first, and it can be computed very fast as in confidence-corrected surprise without needing to maintain the idea of a naive observer.

\subsection{Variational \ainvase{} Surprise for Exploration (VASE)}
\label{sec-vase-impl}

\begin{figure}
    \centering
    \includegraphics[width=0.6\linewidth]{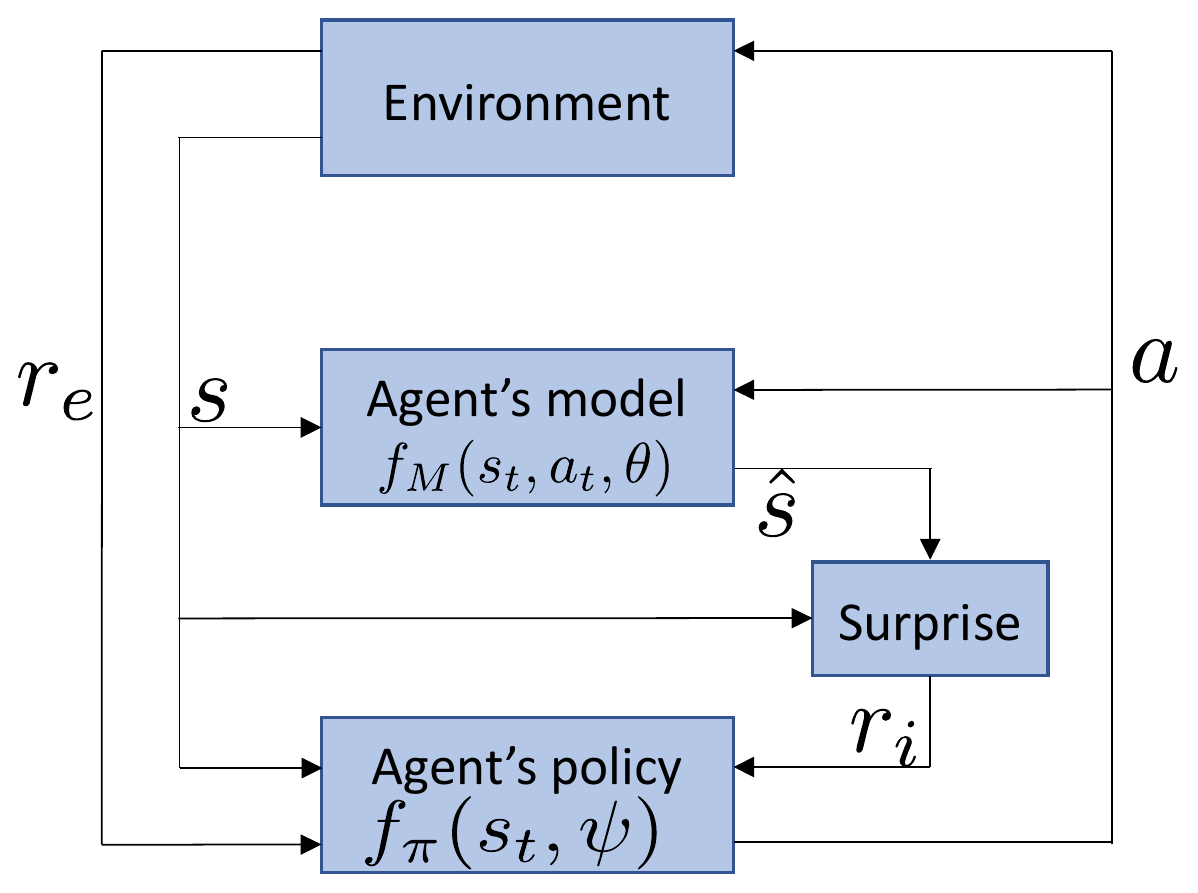}
    \caption{VASE-driven RL with $s$ for state, $r_e$ for the extrinsic (environment-driven) reward, $a$ for action, $r_i$ for intrinsic (agent-driven) reward, $a_t=f_\pi(s_t,\psi)$ the policy governed by set of parameters $\psi$, and $\widehat{s}=f_m(s_t,a_t,\theta)$ the sample BNN model of the environment.}
    \label{fig:rl_detail}
\end{figure}

Based on the discussion in Section \ref{sec-surprise}, we construct a BNN dynamics model $f_m(s_t,a_t,\Theta)$, where $\Theta$ is a random variable describing the parameters of the model (see Figure \ref{fig:rl_detail}).  BNN can be seen as a distribution of models $M$, where a sample of network parameters $\theta$ according to distribution $P(\theta)$ is analogous to generating a single prediction of the next state according to $P(M)$.  The prior distribution $P(\theta)$ changes to posterior $P(\theta|\mathcal{D})$ when BNN is trained by $\mathcal{D} = \{s_t, a_t, s_{t+1}\}$.

Since the posterior  $P(\theta|s_t,a_t)$ in Eq. \ref{eqn_sours} is intractable, we turn to variational inference \cite{bishop2006pattern} to approximate it  with a fully factorised Gaussian distribution \cite{graves2011practical, blundell2015weight, hinton1993keeping}
\begin{eqnarray}\label{eq:full_gauss}
q(\theta;\phi) = \prod_{i=1}^{|\Theta|}\mathcal{N}(\theta_i;\mu_i,\sigma_i^2),
\end{eqnarray}
where $\theta_i$ is the $i^{\text{th}}$ component of $\theta$, and $\phi_i = (\mu_i,\sigma_i)$.

 The use of $q(\theta;\phi)$ in place of $P(\theta|s_t,a_t)$ changes the definition of surprise from Eq. \ref{eqn_sours} to one we call variational \our{} surprise for exploration (VASE): 
\begin{eqnarray}\label{eq:app_sur}
U_{\text{VASE}}(s_{t+1}) & = \mathbb{E}_{\theta \sim q(\cdot;\phi)}[-\log{P(s_{t+1}|s_t,a_t,\theta)}]\nonumber\\
& -\delta H(q(\theta;\phi)).
\end{eqnarray}

Since the output of the model for sample $\theta$ gives the prediction of the next state $\widehat{s}_{t+1}=f_m(s_t,a_t,\theta)$, we define $P(s_{t+1}|s_t,a_t,\theta)$ by measuring the deviation of $\widehat{s}_{t+1}$ from $s_{t+1}$ under the assumption that states are normally distributed:
\begin{equation}\label{eq:app_NLL}
P(s_{t+1}|s_t,a_t,\theta)=\frac{1}{\sqrt{2\pi\sigma_c^2}}e^{-||\widehat{s}_{t+1}-s_{t+1}||^2/(2\sigma_c^2)},
\end{equation}
where  $\sigma_c$ is an arbitrarily chosen constant, $||\widehat{s}_{t+1}-s_{t+1}||$ is the norm of the difference vector between the prediction of the next state and the true next state. Note that this is a slightly different formulation to that in Equation \ref{eq-st1-withM} but approaches that formulation as $\sigma_c$ approaches 0.

N samples of $\theta \sim q(\cdot;\phi)$ give $N$ predictions for the next state from the BNN, which allows us to estimate the first term of Eq. \ref{eq:app_sur} with the average:
\begin{eqnarray*}\label{eqn:exp_lltheta}
\mathbb{E}_{\theta \sim q(\cdot;\phi)}[\log{P(s_{t+1}|s_t,a_t,\theta)}]\\\approx \frac{1}{N}\sum_{n=1}^N\log P(s_{t+1}|s_t,a_t,\theta^{[n]}),
\end{eqnarray*}
where $\theta^{[n]}$ is the $n^{\text{th}}$ sample of $\Theta$ drawn from $q(\theta;\phi)$ and $P(s_{t+1}|s_t,a_t,\theta^{[n]})$ is evaluated according to Eq. \ref{eq:app_NLL}.

Since $q(\theta;\phi)$ is a fully factorised Gaussian distribution, the second term of Eq. \ref{eq:app_sur} is straight forward to evalute:
\begin{eqnarray*}
H\big (q(\theta;\phi)\big )  &=& \sum_{i=1}^{|\Theta|} H(\mathcal{N}(\theta_i;\mu_i,\sigma_i^2)) \\
&=& \frac{1}{2}\sum_{i=1}^{|\Theta|}(\log(2\pi e \sigma_i^2).
\end{eqnarray*}

The last thing remaining is to ensure $q(\theta;\phi)$ is as close as possible to $P(\theta|\mathcal{D})$.  
Variational inference uses Kullback-Leibler (KL) divergence for measuring how different $q(\theta;\phi)$ is from $P(\theta|\mathcal{D})$:
\begin{eqnarray*}
 D_{KL} [q(\theta; \phi)|| P(\theta|\mathcal{D})]
       & = & \int_{\mathcal{\theta}}q(\theta; \phi)\log{\frac{q(\theta; \phi)}{P(\theta|\mathcal{D})}}\partial \theta\nonumber\\
       & = & D_{KL} [q(\theta; \phi)|| P(\theta)]\\
       && -\mathbb{E}_{\theta \sim q(\cdot;\phi)}[\log{P(\mathcal{D}|\theta)}] + \log{P(\mathcal{D})}.
       %& = & - L[q(M; \phi), \mathcal{D}] + \log{P(\mathcal{D})},
\end{eqnarray*}
 %We can see that $L[q(M; \phi), \mathcal{D}] = \log{P(\mathcal{D})} -D_{KL} [q(M; \phi)|| P(M|\mathcal{D})]$. 

This difference is minimised by changing $\phi$, which is equivalent to maximising the variational lower bound \cite{bishop2006pattern}:
\begin{eqnarray}\label{eq:elbo}
L[q(\theta; \phi), \mathcal{D}] = \mathbb{E}_{\theta \sim q(\cdot;\phi)}[\log{P(\mathcal{D}|\theta)}]-D_{KL} [q(\theta; \phi)|| P(\theta)],
\end{eqnarray}
which does not require evaluation of $P(\theta|\mathcal{D})$.  In this paper, the prior distribution of $\theta$ is taken to be 
\begin{equation}\label{eq:prior_theta_gauss}
P(\theta)=\prod_{i=1}^{|\Theta|}\mathcal{N}(\theta_i;0,\sigma_m^2), 
\end{equation}
where $\sigma_m$ is set to arbitrary value, and the expectation of log likelihood of $P(\theta|\mathcal{D})$ is evaluated as in Eq. \ref{eqn:exp_lltheta}.

%\begin{eqnarray}
%r(s_t , a_t , s_{t+1}) = r_e(s_t, a_t) + \eta U_{\text{VASE}}.
%\end{eqnarray}
%The RL algorithm will choose this reward to train the agent. If there is no extrinsic reward, the agent will be trained directly by $U_{\text{VASE}}$.

 The entire training procedure is listed in Algorithm \ref{alg::VASE}. 
 
 %In this paper, we choose Trust Region Policy Optimization (TRPO) \cite{schulman2015trust} as our RL policy update algorithm in all of our experiments. However, it should be noted that our surprise-driven method can be embedded into any other RL algorithms. 

\begin{algorithm}
\DontPrintSemicolon
\SetAlgoLined
\SetKwInOut{Input}{input}\SetKwInOut{Output}{output}
{Initialise policy neural network $f_{\pi}$ with parameters $\psi$ \;
Initialise agent's BNN model $f_m$:\;
\hspace*{0.5cm} Initialise $q(\theta;\phi)$ with parameters $\phi$\;
 \hspace*{0.5cm} Initialise prior distribution $P(\theta)$\; 
 %Initial state distribution $\rho_0$\;}
 Initialise experience buffer $R$.\;
 Reset the environment getting $(s_0,r_e(s_0))$.\;}
\For{each iteration $n$}
{
    \For{each time step $t$}
    {
    	Get action $a_t \sim f_{\pi}(s_t,\psi)$\;
    	Sample $\theta$ $N$ times according to $q(\theta;\phi)$\;
    	Evaluate $N$ predictions $\widehat{s}_{t+1}=f_m(s_t,a_t,\theta)$\;
    	Take action $a_t$ getting $(s_{t+1},r_e(s_{t+1}))$\;
    	%Get $r_e(s_t , a_t)$ from environment\;
    	%\If{not enough samples in $R$}
		%{
		%Add new $(s_t, a_t, s_{t+1})$ triplet to $R$\;
		%}
		Compute intrinsic surprise $U_{\text{VASE}}(s_{t+1})$ \;
		Construct cumulative reward $r(s_{t+1}) = r_e(s_{t+1}) + \eta U_{\text{VASE}}(s_{t+1})$\;
		Add new $(s_t,a_t,s_{t+1},r(s_{t+1}))$ to $R$\;
	}
	Update $f_m$ by maximising Eq. (\ref{eq:elbo}), with $\mathcal{D}$ sampled randomly from $R$\;
	Update $f_{\pi}$ using TRPO.
}
\caption{Variational \ainvase{} Surprise Exploration (VASE)}\label{alg::VASE}
\end{algorithm}

\section{Experiments and results analysis}

\subsection{Visualising exploration efficiency}

\begin{figure*}
    \centering
    \subfigure[TRPO (2,059,459 steps)]
    {
        \includegraphics[width=0.45\linewidth]{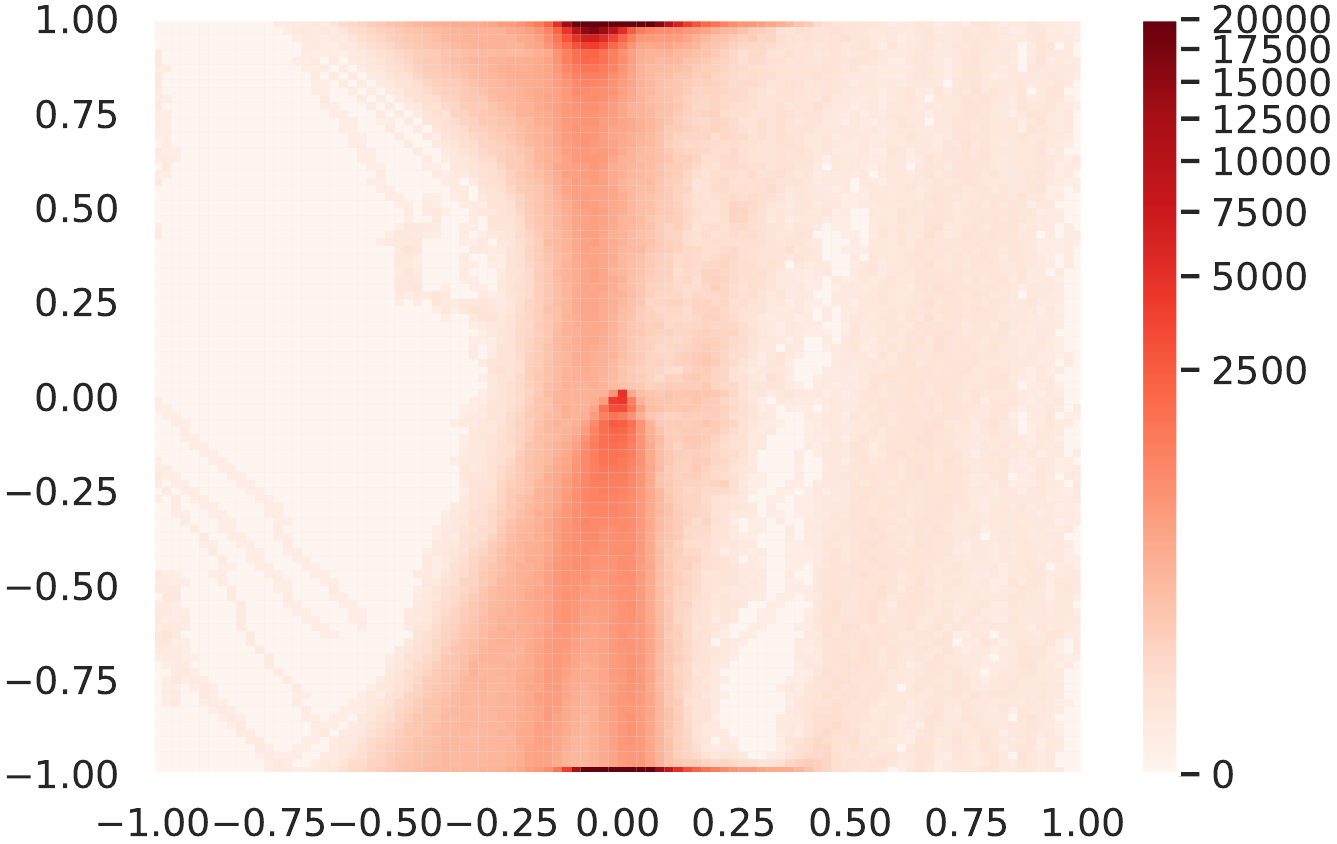}
        \label{fig:exp_strategy_nosup}
    }
    \subfigure[TRPO+VASE (26,663 steps)]
    {
        \includegraphics[width=0.45\linewidth]{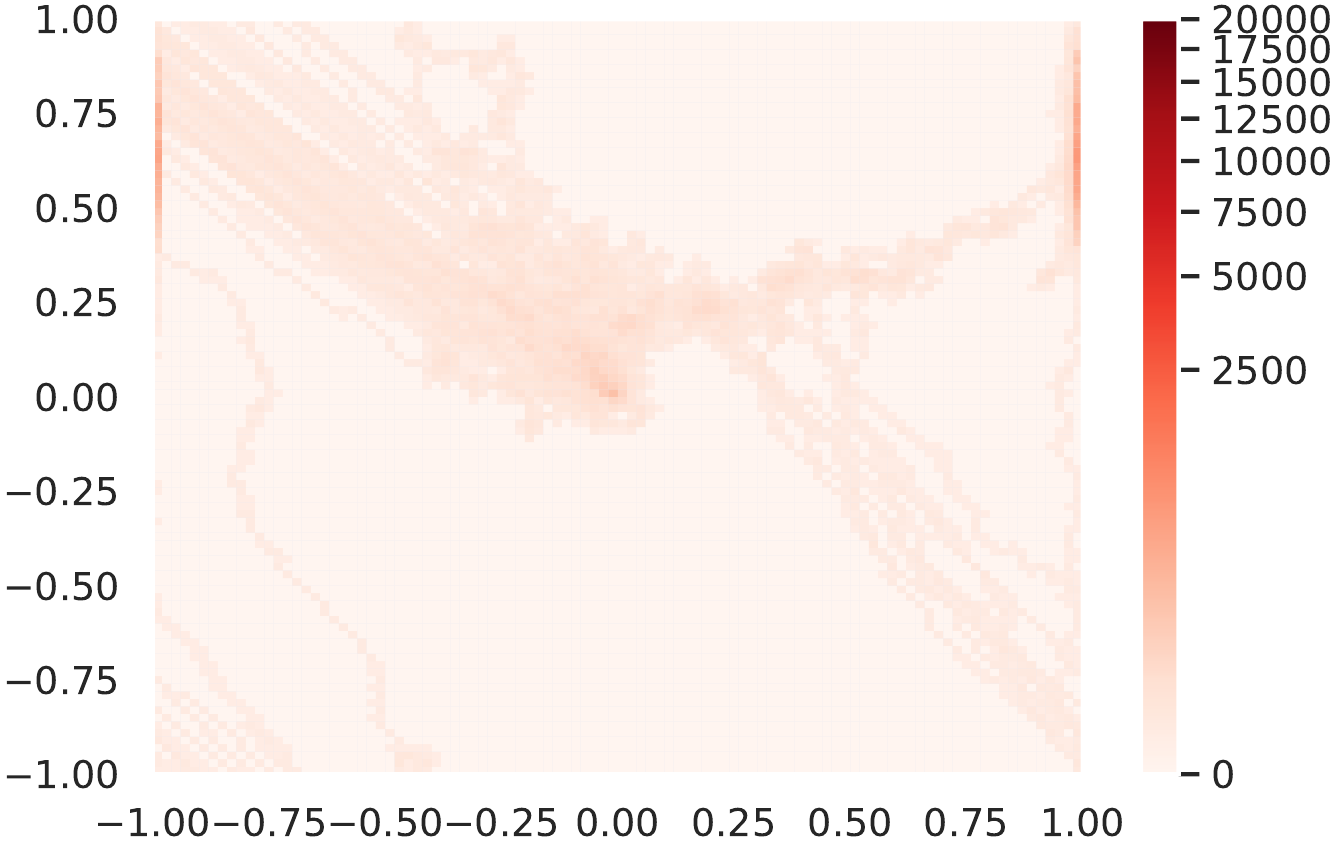}
        \label{fig:exp_strategy_sup}
    }
    \caption{Exploration efficiency as a heatmap showing the number of states visited during training in 2DPointRobot environment until chancing upon the reward state with a) no surprise, b) VASE.}
    \label{fig:exp_strategies}
\end{figure*}

For illustrative purposes, we begin the experimental evaluation of VASE by testing it on a simple 2DPlane environment ($\mathcal{S}\subset \mathbb{R}^2, \mathcal{A}\subset \mathbb{R}^2$) which lends itself to a visualisation of the agent's exploration efficiency.
% of a 80 x80 discrete states and 4 possible actions for moving north, east, south, or west to an adjacent state ($\mathcal{S}\subset \in \{(x,y): x=\pm n/40, y=\pm n/40, \forall n=-40,39,...,39,40\}, \mathcal{A}\subset \{N,E,S,W\}$).   
The observation space is a square on the 2D plane $((x, y) \in \mathbb{R}^2)$, centred on the origin. The action is its velocity $(\dot x, \dot y)$ that satisfies $|\dot x| \leq 0.01, |\dot y| \leq 0.01$. 
In this environment the agent starts at origin (0,0) and the only extrinsic reward can be found at location (1,1).  The environment wraps around, so that there are no boundaries.

In this experiment, we train one agent and record the observation coordinate $(x,y)$ in each step until it finds the non zero extrinsic reward.  Figure \ref{fig:exp_strategies} shows the heat map of motion track for the agent trained without surprise and with VASE surprise. Darker red colour represents a higher density, which means the agent lingers more steps in this area. It is clear that random exploration strategy takes a long time (2,059,459 steps V.S. 26,663 steps) to find that first non-zero $r_e(s_t)$, whereas VASE does not spend time unnecessarily in random states.

\subsection{Continuous state/action environments}

\begin{figure*}
    \centering
    \subfigure[MountainCar]
    {
        \includegraphics[width=0.45\linewidth]{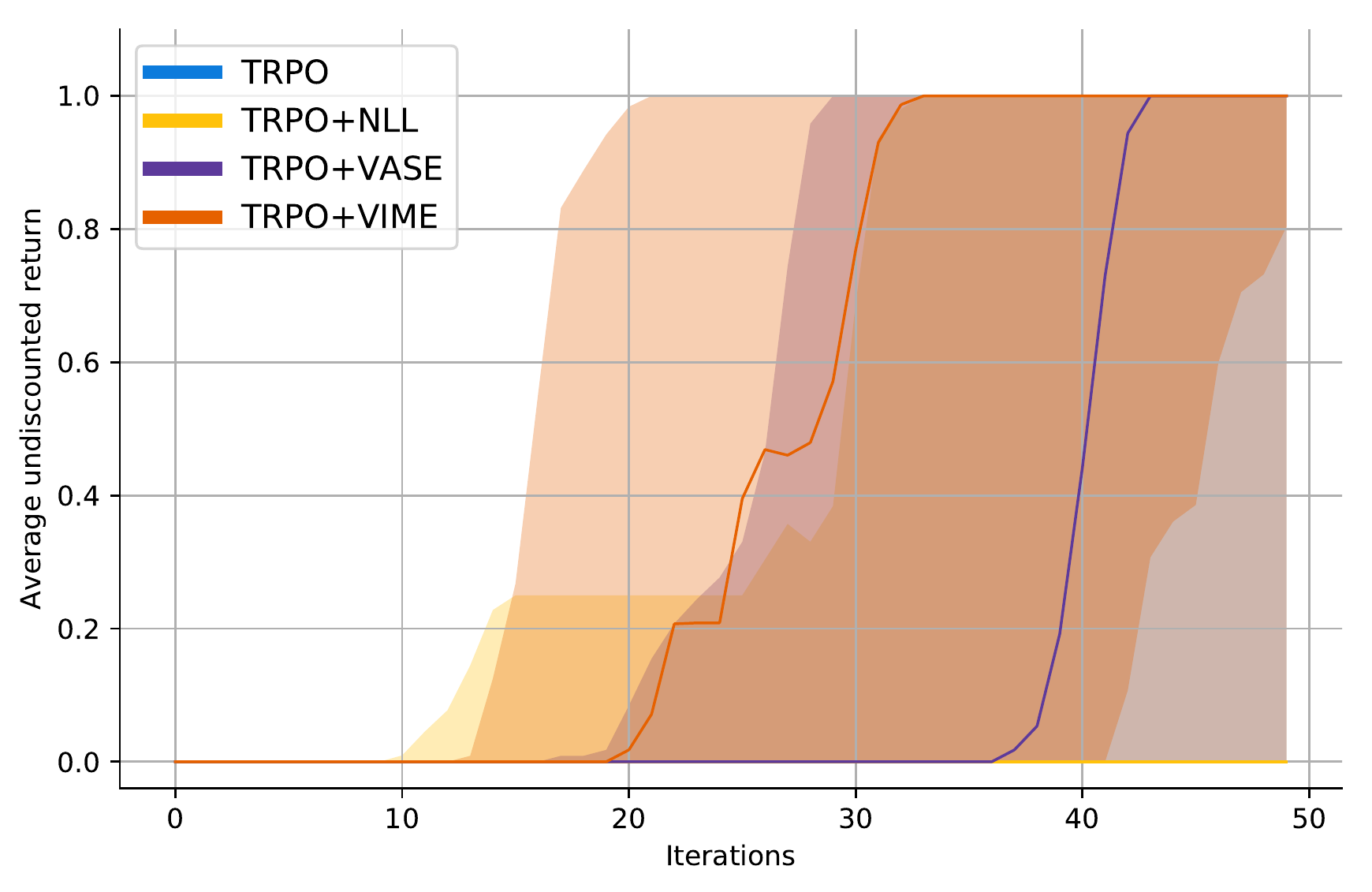}
        \label{fig:MountainCar-m}
    }
    \subfigure[CartpoleSwingup]
    {
        \includegraphics[width=0.45\linewidth]{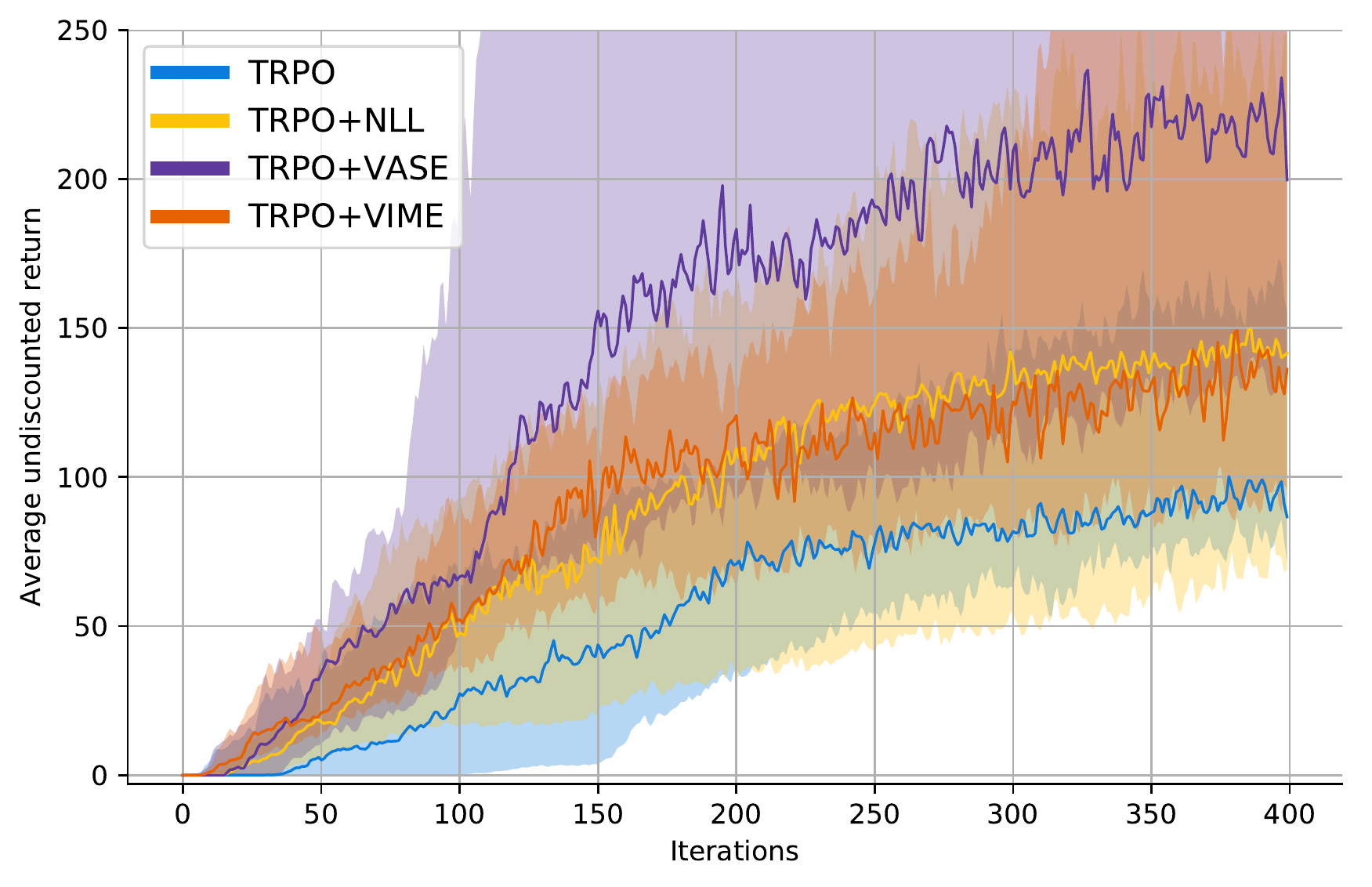}
        \label{fig:CartpoleSwingup-m}
    }
    \\
    \subfigure[DoublePendulum]
    {
        \includegraphics[width=0.45\linewidth]{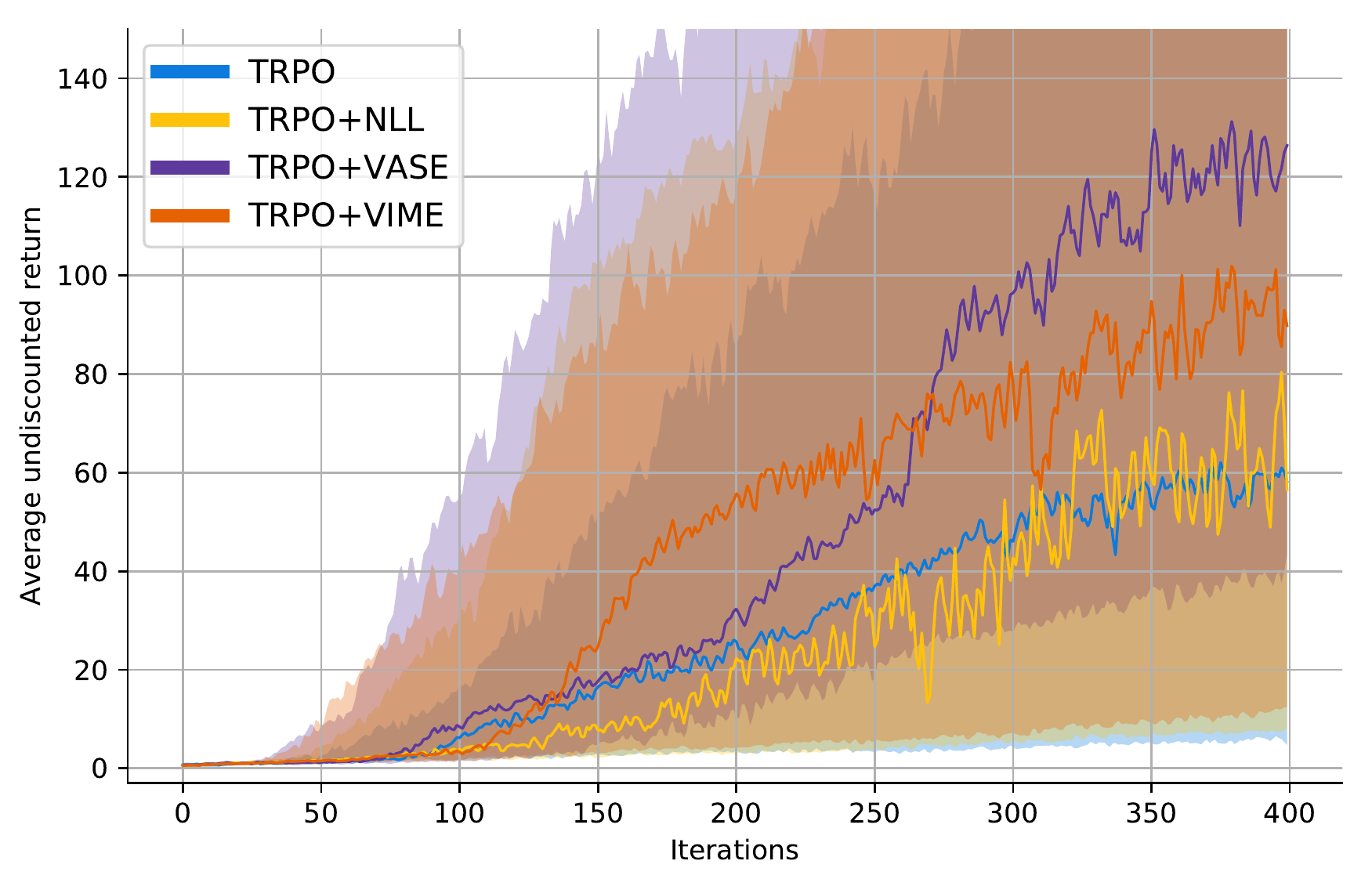}
        \label{fig:DoublePendulum-m}
    }
    \subfigure[HalfCheetah]
    {
        \includegraphics[width=0.45\linewidth]{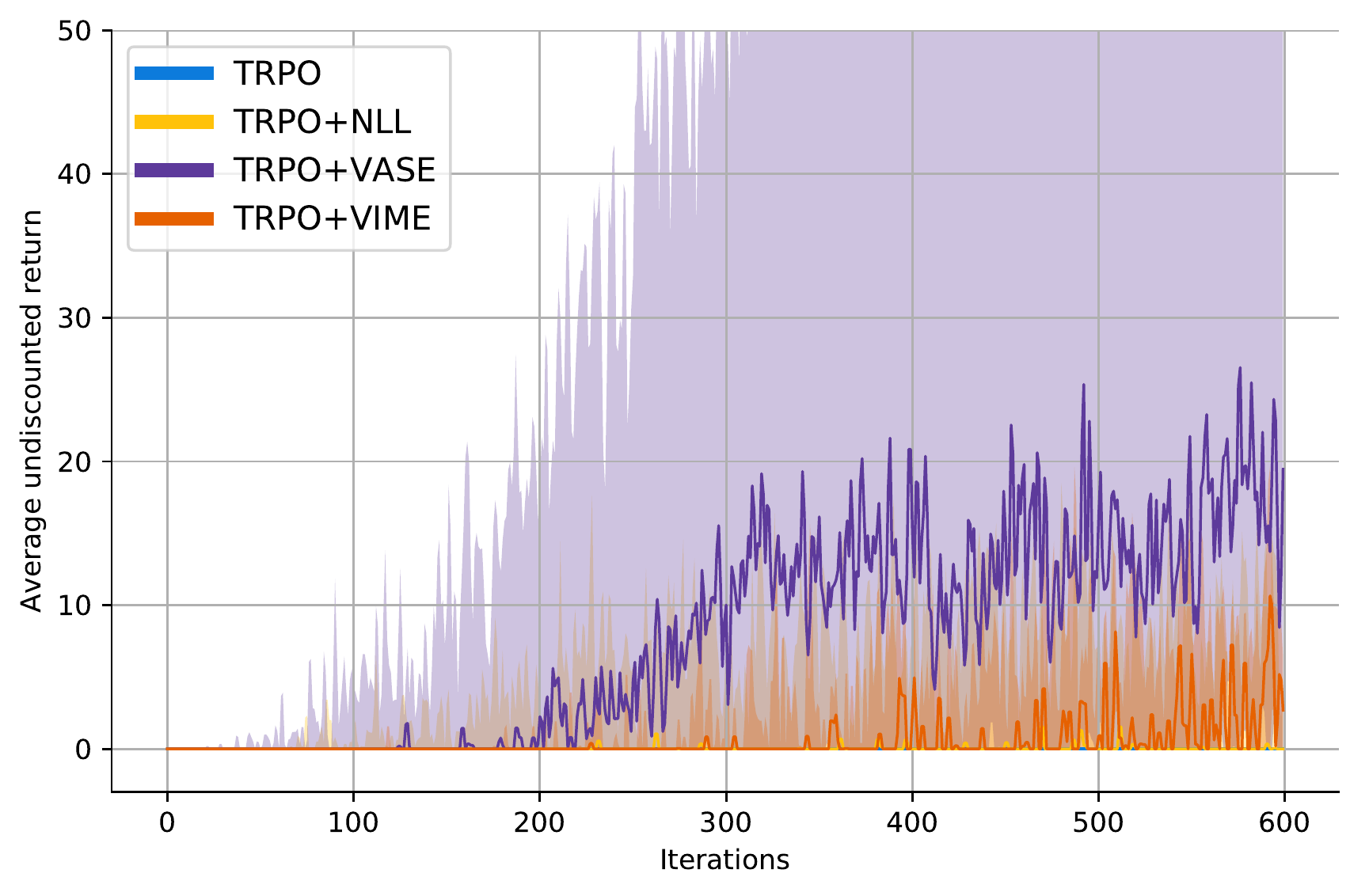}
        \label{fig:HalfCheetah-m}
    }
    \subfigure[Ant]
    {
        \includegraphics[width=0.45\linewidth]{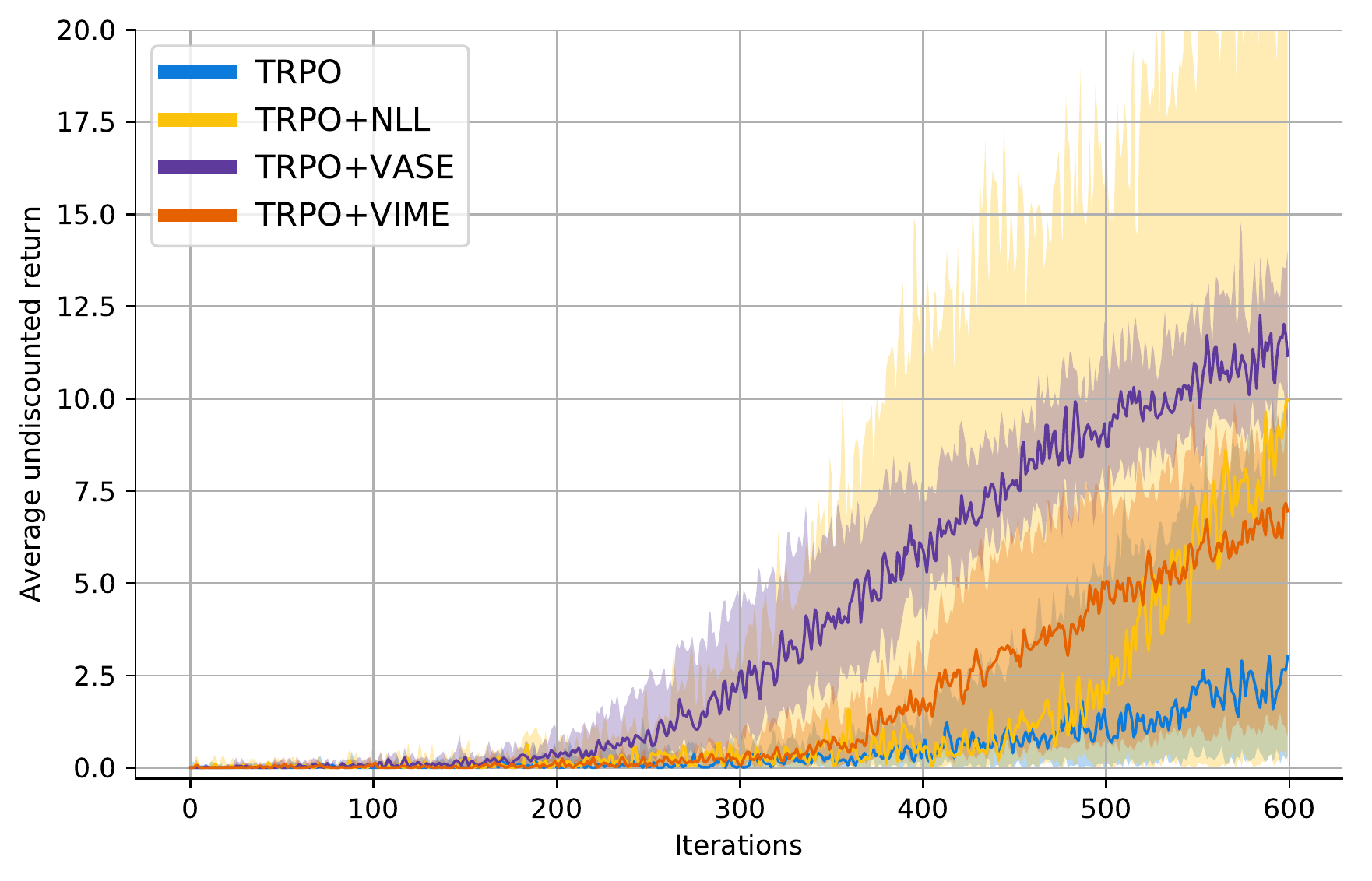}
        \label{fig:ant-m}
    }
    \subfigure[Lunar-lander]
    {
        \includegraphics[width=0.45\linewidth]{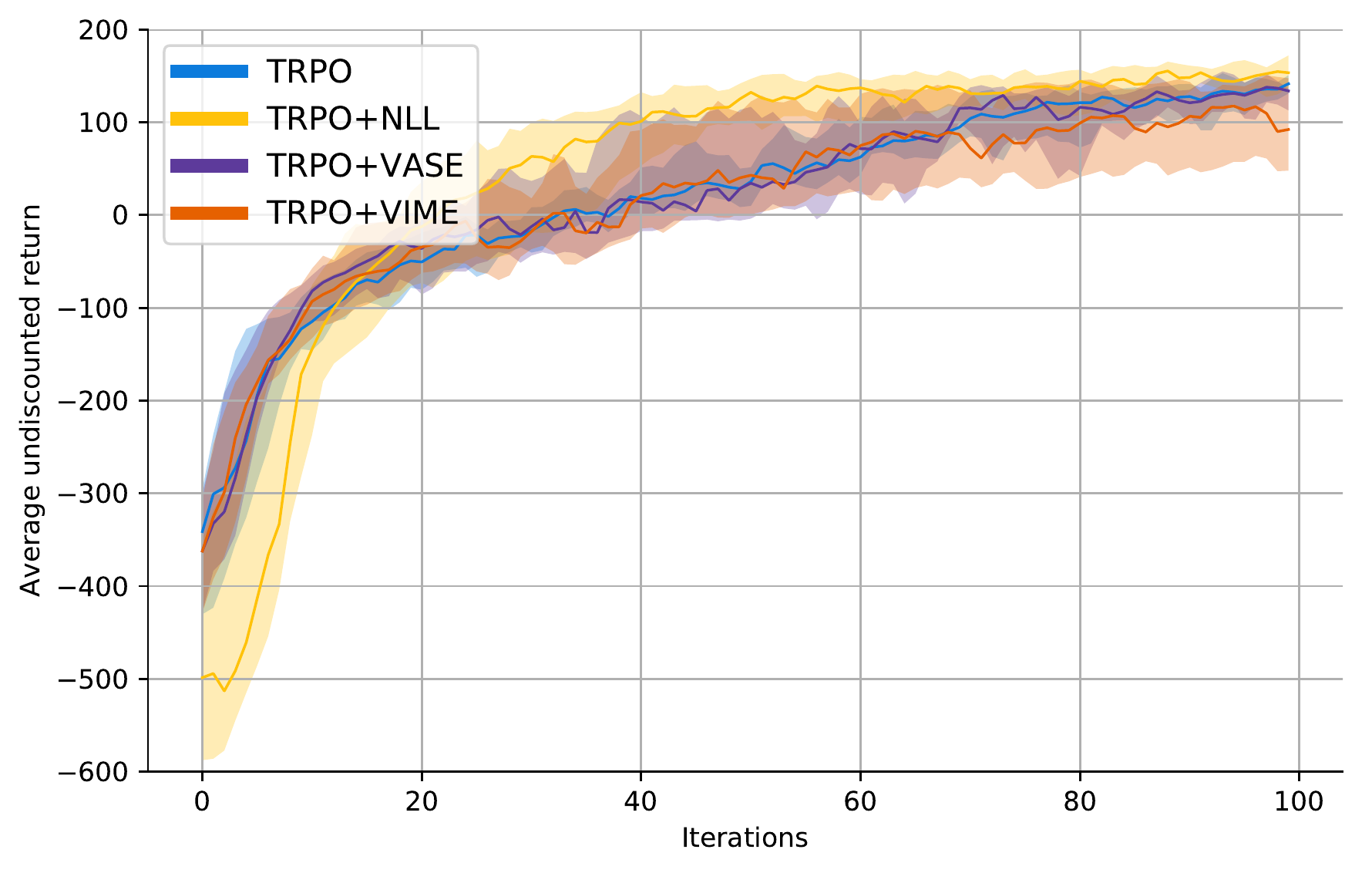}
        \label{fig:lunar-lander-m}
    }
    \caption{Median performance for the continuous control tasks over 20 runs with a fixed set of seeds,
with interquartile ranges shown in shaded areas. VIME, NLL and VASE use Bayesian surprise, surprisal, our surprise respectively.}
    \label{fig:Median}
\end{figure*}

Next we evaluate VASE on five continuous control benchmarks with very sparse reward, including three classic tasks: sparse MountainCar ($\mathcal{S}\subset \mathbb{R}^3, \mathcal{A}\subset \mathbb{R}^1$), sparse CartPoleSwingup ($\mathcal{S}\subset \mathbb{R}^4, \mathcal{A}\subset \mathbb{R}^1$), sparse Doublependulum ($\mathcal{S}\subset \mathbb{R}^6, \mathcal{A}\subset \mathbb{R}^1$) and two locomotion tasks: sparse HalfCheetah ($\mathcal{S}\subset \mathbb{R}^{20}, \mathcal{A}\subset \mathbb{R}^6$), sparse Ant ($\mathcal{S}\subset \mathbb{R}^{125}, \mathcal{A}\subset \mathbb{R}^8$). These tasks were introduced in \cite{houthooft2016vime}. We also evaluate VASE on LunarLanderContinuous ($\mathcal{S}\subset \mathbb{R}^{8}, \mathcal{A}\subset \mathbb{R}^2$) task.  

For the sparse MountainCar task, the car will climb a one-dimensional hill to reach the target. The target is located on top of a hill and on the right-hand side of the car. If the car reaches the target, the episode is done. The observation is given by the horizontal position and the horizontal velocity of the car. The agent car receives a reward of 1 only when it reaches the target. 

For the sparse CartpoleSwingup task, a pole is mounted on a cart. The cart itself is limited to linear motion. Continuous cart movement is required to keep the pole upright. The system should not only be able to balance the pole, but also be able to swing it to an upright position first. The observation includes the cart position $x$, pole angle $\beta$, the cart velocity $\dot x$, and the pole velocity $\dot \beta$. The action is the horizontal force applied to the cart. The agent receives a reward of 1 only when $cos(\beta) > 0.9$.  

For the sparse Doublependulum task, the goal is to stabilise a two-link pendulum at the upright position. The observation includes joint angles $(\beta_1$ and $\beta_2)$ and joint speeds $(\dot \beta_1$ and $\dot \beta_2)$.  The action is the same as in CartpoleSwingup task. The agent receives a reward of 1 only when dist $< 0.1$, with dist the distance between current pendulum tip position and target position. 

For the sparse HalfCheetah task, the half-cheetah is a flat biped robot with nine rigid links, including two legs and one torso, and six joints. 20-dimensional observations include joint angle, joint velocity, and centroid coordinates. The agent receives a reward of $1$ when $x_{body} \geq 5$. 

For the sparse Ant task, the ant has 13 rigid links, including four legs and a torso, along with 8 actuated joints.  The 125-dim observation includes joint angles, joint velocities, coordinates of the centre of mass, a (usually sparse) vector of contact forces, as well as the rotation matrix for the body. The ant receives a reward of $1$ when $x_{body} \geq 3$.

For Lunar-lander task, the agent tries to learn to fly and then land on its landing pad. The episode is done if the lander crashes or comes to rest. The agent should get rewards of $200$ when it solves this task. 

All the environments except MountainCar and 2DPointRobot tasks (they are two simple tasks, we do not need to normalise them) are normalised before the algorithm starts. Here normalise the task means normalise its observations, for each observation $o$:
$$o = \frac{(o - \mu_o)}{\sigma_o},$$ where $\mu_o$ and $\sigma_o$ are the mean and standard deviation of observations. All observations and actions in these environments are continuous values. To compare with \citep{achiam2017surprise} and \citep{houthooft2016vime}, we also use Trust Region Policy Optimization (TRPO) \cite{schulman2015trust} method as our base reinforcement learning algorithm throughout our experiments, and we use the rllab \cite{duan2016benchmarking} implementations of TRPO.

The number of samples drawn to compute our surprise is $N=10$. The prior distribution $P(\theta)$ is given by a Gaussian distribution from Eq. \ref{eq:prior_theta_gauss} with $\sigma_m=0.5$. $\sigma_c$ in Eq. \ref{eq:app_NLL} is set as 5. For the classic tasks sparse MountainCar, sparse CartPoleSwingup, sparse DoublePendulum and sparse LunarLanderContinuous, the $f_m$ has one hidden layer of 32 units. All hidden layers have rectified linear unit (ReLU) non-linearities. The replay pool $R$ has a fixed size of 100,000 samples, with a minimum size of 500 samples. For the locomotion tasks sparse HalfCheetah and sparse Ant, the $f_m$ has two hidden layers of 64 units each. All hidden layers have tanh non-linearities. The replay pool $R$ has a fixed size of 5,000,000 samples. The Adam learning rate of $f_m$ is set to 0.001. All output layers are set to linear. 
The batch size for the policy optimisation is set to 5,000. For $f_\pi$ the classic tasks use a neural network with one layer of 32 tanh units, while the locomotion task uses a two-layer neural network of 64 and 32 tanh units. For baseline, the classic tasks use a neural network with one layer of 32 ReLU units, while the locomotion task uses a linear function. The maximum length of trajectory LunarLanderContinuous 1000, for all the other tasks, it is 500. 

Figure \ref{fig:Median} (a)-(e) shows the median performance of three classic control tasks and two locomotion tasks. All these tasks are with sparse rewards.  Figure \ref{fig:Median} (f) shows the median performance of LunarLanderContinuous task. The agent can easily obtain rewards from this task. The performance is measured through the average return $\mathbb{E}_{\tau}[\sum_{t=0}^{T}r_e(s_t)]$, not including the intrinsic rewards.  The median performance curves with shaded interquartile ranges areas. Figure \ref{fig:MountainCar-time} shows the speed comparison on MountainCar task.

As can be seen from Figure \ref{fig:Median} (a)-(e), VIME performs best for sparse MountainCar task. For the sparse DoublePendulum, VIME performs well initially, but is later surpassed by VASE. VASE shows good results in sparse CartpoleSwingup, sparse HalfCheetah and sparse Ant tasks. We can also see that VASE always performs better than NLL (suprisal) in all tasks. Figure \ref{fig:Median} (f) shows that in LunarLanderContinuous task that has enough reward for the agent, all surprise-driven methods behave almost the same to the no-surprise method.

Figure \ref{fig:MountainCar-time} shows us the speed test results. For VIME, it needs to calculate a gradient through its BNN at each time step to compute the Bayesian surprise reward. This is really time consuming. However, for our VASE algorithm, it does not need to compute this gradient. Figure \ref{fig:MountainCar-time} shows that VASE runs much faster than VIME.

\begin{figure}
    \centering
	%\subfigure[Mean]
    {
        \includegraphics[width=0.85\linewidth]{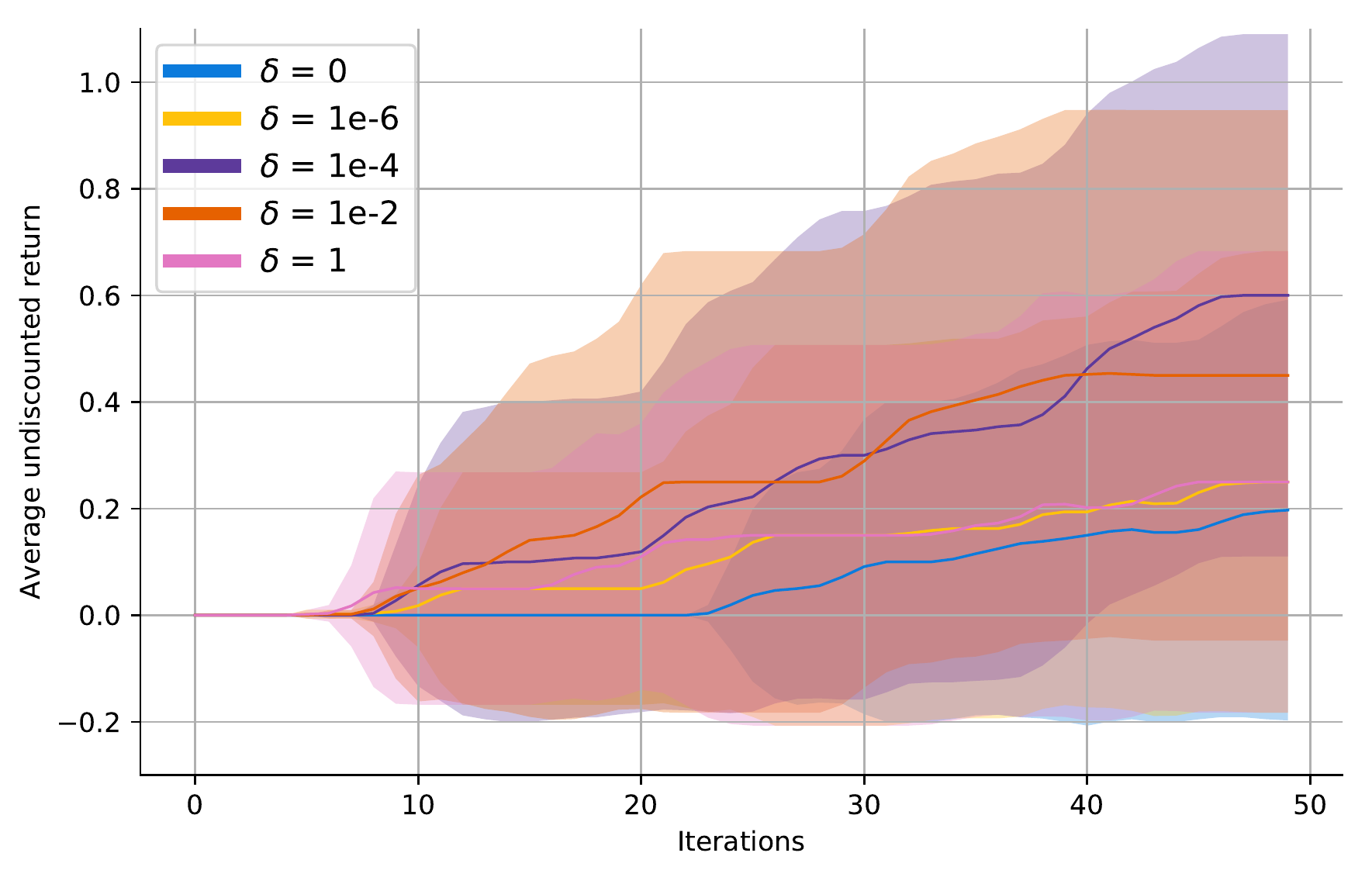}        
    } 
    \caption{Mean performance for MountainCar task with different $\delta$.}\label{fig:MountainCar-delta}
\end{figure}
\begin{figure}
    \centering
	%\subfigure[Mean]
    {
        \includegraphics[width=0.85\linewidth]{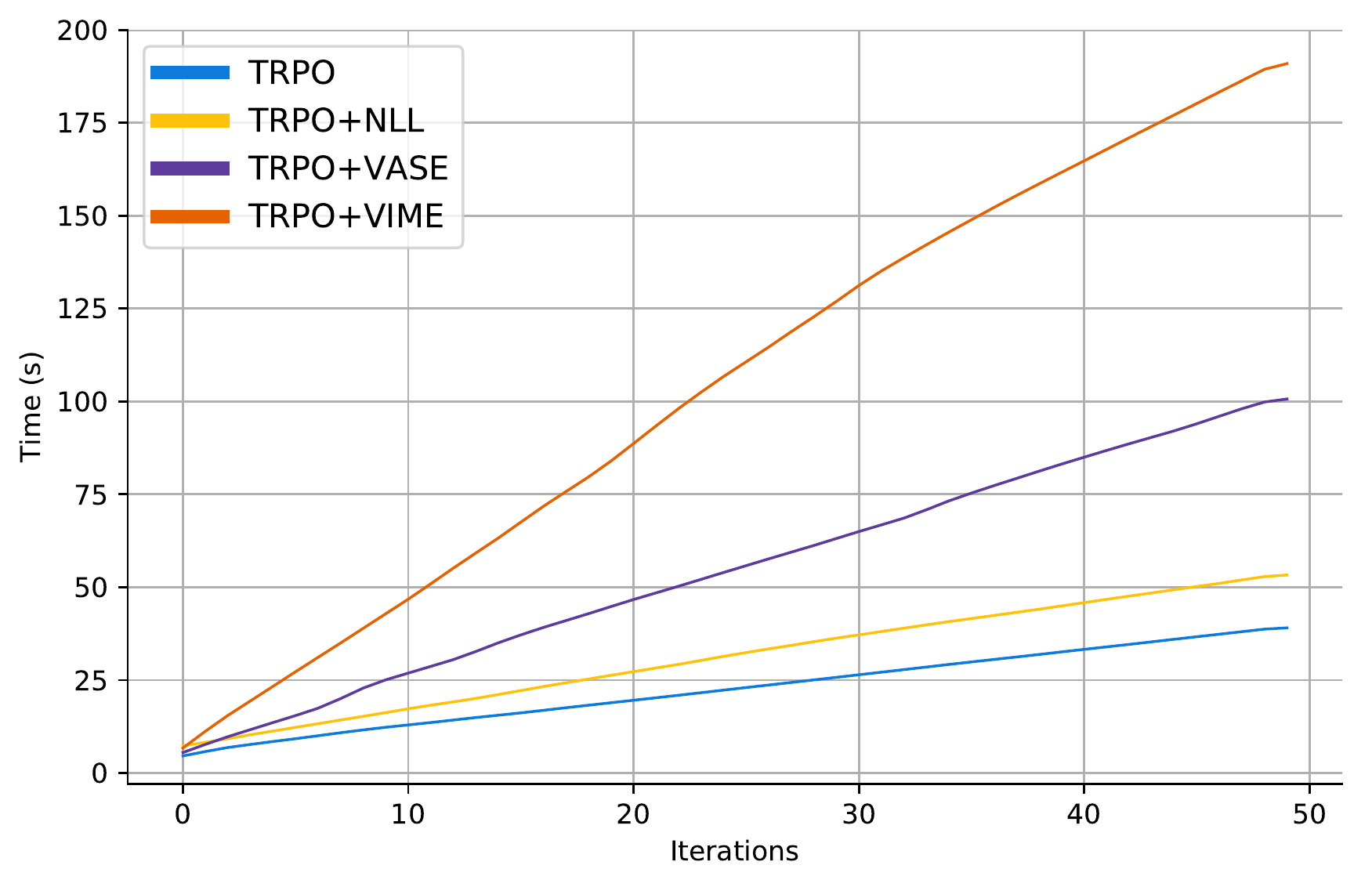}        
    } 
    \caption{Running time on MountainCar environment.}\label{fig:MountainCar-time}
\end{figure}

Finally, we also test how different values of trade-off $\delta$ that we used in Eq. \eqref{eq:app_sur} affects the performs of surprise $U_{\text{VASE}}$ on sparse MountainCar environment. We know that the value of $H(q(M;\phi))$ depends not only on the distribution of each parameter $M_i$, but also on the number of parameters $|M|$. Meanwhile, in the beginning stages of training, the entropy of each parameter $M_i$ is relatively large, therefore, we should take a relatively small value of $\delta$. Figure \ref{fig:MountainCar-delta} shows VASE performance based on $\delta$ chosen from \{0, 1e-6, 1e-4, 1e-2, 1\}. The performance is not good when $\delta$ is too big or too small. It shows that the best interval to search $\delta$ is [1e-4,1e-2]. 
%%---------------------------------------------------------------
\section{Conclusions}
In this work, we chose a new form of surprise as the agent's intrinsic motivation and applied it to the RL settings by our VASE algorithm.  VASE tries to approximate this surprise in a tractable way and train the agent to maximise its reward function. The agent is driven by this intrinsic reward, which can effectively explore the environment and find sparse extrinsic rewards given by the environment. Empirical results show that VASE performs well across various continuous control tasks with sparse rewards. We believe that VASE can be easily extended to deep reinforcement learning methods or learn directly from pixel features. We leave that to future work to explore.

\bibliographystyle{unsrtnat} 
\bibliography{VASE}

%\begin{appendices}
%\onecolumn
%%\appendix
%\section{Eq. (4)}
%\begin{eqnarray*}\label{suprise-sum}
% &&\mathbb{E}_{M \sim P(\cdot|s_t,a_t)}[-\log{P(s_{t+1}|s_t,a_t,M)}] \\ 
% & = & -\int_{M}P(M|s_t,a_t)\log{P(s_{t+1}|s_t,a_t,M)}dM\\
%& = & -\int_{M}P(M|s_t,a_t) \log{
%    \frac{P(M|s_t,a_t,s_{t+1})P(s_{t+1}|s_t,a_t)}{P(M|s_t,a_t)}}dM\\
%& = & \int_{M}P(M|s_t,a_t) \log{
%    \frac{P(M|s_t,a_t)}{P(M|s_t,a_t,s_{t+1})}}dM -
%    \int_{M}P(M|s_t,a_t)\log{P(s_{t+1}|s_t,a_t)}dM \\
%& = & D_{KL}[P(M|s_t,a_t)||P(M|s_t,a_t,s_{t+1})] + (- \log{P(s_{t+1}|s_t,a_t)}),
%\end{eqnarray*} 
%which is a sum of Bayesian surprise (belief dependent) and surprisal (belief independent) when a new state $s_{t+1}$ is oberseved.
%
%\section{Experiments}
%\end{appendices}
\end{document}